\documentclass[sn-apa]{sn-jnl}

\usepackage{graphicx}%
\usepackage{multirow}%
\usepackage{amsmath,amssymb,amsfonts}%
\usepackage{amsthm}%
\usepackage[title]{appendix}%
\usepackage{xcolor}%
\usepackage{textcomp}%
\usepackage{manyfoot}%
\usepackage{booktabs}%
\usepackage{algorithm}%
\usepackage{algorithmicx}%
\usepackage{algpseudocode}%
\usepackage{listings}%

\usepackage{float}
\usepackage[section]{placeins}
\usepackage{graphicx}
\usepackage[caption=false]{subfig}
\usepackage{lmodern}
\usepackage{natbib} 
\usepackage{graphicx}

\newtheorem{theorem}{Theorem}
\newtheorem{proposition}[theorem]{Proposition}%
\newtheorem{lemma}{Lemma}

\newtheorem{definition}{Definition}%

\begin{document}

\title[Article Title]{Towards Understanding the Optimization Mechanisms in Deep Learning}


\author*[1]{\fnm{Binchuan} \sur{Qi}}\email{2080068@tongji.edu.cn} 

\author*[1,2,3]{\fnm{Wei} \sur{Gong}}\email{weigong@tongji.edu.cn}

\author[1,2,3]{\fnm{Li} \sur{Li}}\email{lili@tongji.edu.cn}

\affil[1]{\orgdiv{College of Electronics and Information Engineering}, \orgname{Tongji University}, \orgaddress{\city{Shanghai}, \postcode{201804}, \country{China}}}

\affil[2]{\orgdiv{National Key Laboratory of Autonomous Intelligent Unmanned Systems}, \orgname{Tongji University}, \orgaddress{\city{Shanghai}, \postcode{201210}, \country{China}}}

\affil[3]{\orgdiv{Shanghai Research Institute for Intelligent Autonomous Systems}, \orgname{Tongji University}, \orgaddress{\city{Shanghai}, \postcode{201210}, \country{China}}}

\abstract{
In this paper, we adopt a probability distribution estimation perspective to explore the optimization mechanisms of supervised classification using deep neural networks. 
We demonstrate that, when employing the Fenchel-Young loss, despite the non-convex nature of the fitting error with respect to the model's parameters, global optimal solutions can be approximated by simultaneously minimizing both the gradient norm and the structural error. 
The former can be controlled through gradient descent algorithms. For the latter, we prove that it can be managed by increasing the number of parameters and ensuring parameter independence, thereby providing theoretical insights into mechanisms such as over-parameterization and random initialization.
Ultimately, the paper validates the key conclusions of the proposed method through empirical results, illustrating its practical effectiveness.
}


\maketitle
\section{Introduction}
\label{sec:intro}
Alongside the remarkable practical achievements of deep learning, the optimization mechanisms regarding deep learning remain unanswered within the classical learning theory framework.
Key insights from the studies \cite{Yun2018SmallNI, Du2018GradientDP, Chizat2018OnLT, Arjevani2022AnnihilationOS} emphasize the pivotal role of over-parameterization in finding the global optimum and enhancing the generalization ability of deep neural networks (DNNs).
Recent work has shown that the evolution of the trainable parameters in continuous-width DNNs during training can be captured by the neural tangent kernel (NTK)~\cite{Jacot2018NeuralTK,Du2018GradientDF,Zou2018StochasticGD,Arora2019FineGrainedAO,mohamadi2023fast,wang2023ntk}. 
An alternative research direction attempts to examine the infinite-width neural network from a mean-field perspective~\citep{Sirignano2018MeanFA,Mei2018AMF,Chizat2018OnTG,nguyen2023rigorous}. 
However, in practical applications, neural networks are of finite width, and under this condition, it remains unclear whether NTK theory and mean-field theory can adequately characterize the convergence properties of neural networks~\cite{DBLP:conf/msml/SeleznovaK21}.
Therefore, the mechanisms of non-convex optimization in deep learning, and the impact of over-parameterization on model training, remain incompletely resolved.

\textbf{Motivation}. A substantial proportion of machine learning tasks can be conceptualized within the domain of probability distribution estimation. Supervised classification and regression tasks are mathematically characterized by the process of learning the conditional probability distribution of labels given input features. Generative learning involves estimating the underlying joint distribution of features. Consequently, adopting a perspective of probability distribution estimation to analyze deep learning is a logical and coherent choice, and it is applicable to various types of learning tasks. One challenge in analyzing the optimization mechanisms in deep learning lies in the diversity of loss functions used, which leads to different characteristics of the optimization objectives. This makes it difficult to handle them within a unified framework. However, \citet{Blondel2019LearningWF} demonstrated that most loss functions currently used in practical applications can be expressed in the form of Fenchel-Young loss. Therefore, using Fenchel-Young loss is a reasonable and effective approach to reduce the complexity of analysis. 

\textbf{Contribution}. The contributions of this paper are stated as follows. 

\begin{enumerate}
\item This paper proves that when using the Fenchel-Young loss, classification problems in machine learning are equivalent to conditional probability distribution estimation given the features. Additionally, this paper demonstrates that the Fenchel-Young loss possesses implicit regularization capabilities, thereby justifying its use as a standalone optimization objective.
\item This paper proves that although the problem of fitting and learning conditional probability distributions constitutes a non-convex optimization problem when using DNNs, its global optimum is equivalent to its stationary points. Specifically, we can approximate the global optimal solution by reducing both the gradient norm and structural error. This conclusion elucidates the non-convex optimization mechanism underlying model training in deep learning.
\item This paper demonstrates that under the assumption of gradient independence, the structural error is controlled by the number of model parameters; that is, the larger the number of parameters, the smaller the corresponding structural error. This conclusion provides theoretical insights into techniques such as over-parameterization, random parameter initialization, and dropout.
\item Key conclusions drawn from this framework are validated through experimental results.
\end{enumerate}

\textbf{Organization}. The paper is organized as follows:Section~\ref{sec:related_work} reviews the research progress and current status. Fundamental concepts are defined in Section~\ref{sec:pre}. The proposed method and conclusions are detailed in Section~\ref{sec:gd_framework}. Experimental settings and results are presented in Section~\ref{sec:experiment}. The conclusions are summarized in Section~\ref{sec:conclusion}. 

\section{Related Work}
\label{sec:related_work}
Despite the inherent non-convexity of the objective functions, empirical evidence indicates that gradient-based methods, such as stochastic gradient descent (SGD) are capable of converging to the global minima in these networks. 

Existing studies \cite{Du2018GradientDP, Chizat2018OnLT, Arjevani2022AnnihilationOS} indicate that over-parameterization plays a pivotal role in finding the global optimum and enhancing the generalization ability of DNNs.
Therefore, analyzing the optimization mechanisms of deep learning from the perspective of model parameter scale has become an important research direction.
The NTK thus has emerged as a pivotal concept, as it captures the dynamics of over-parameterized neural network trained by GD~\cite{Jacot2018NeuralTK}. 
It is already known in the literature that DNNs in the infinite width limit are equivalent to a Gaussian process~\cite{Neal1996PriorsFI,Williams1996ComputingWI,Winther2000ComputingWF,Neal1995BayesianLF,Lee2017DeepNN}.
The work of~\citet{Jacot2018NeuralTK} elucidates that the evolution of the trainable parameters in continuous-width DNNs during training can be captured by the NTK. Some work has shown that with a specialized scaling and random initialization, the parameters of continuous width two-layer DNNs are restricted in an infinitesimal region around the initialization and can be regarded as a linear model with infinite dimensional features~\cite{Du2018GradientDP,Li2018LearningON,Du2018GradientDF,Arora2019FineGrainedAO}. 
Since the system becomes linear, the dynamics of GD within this region can be tracked via properties of the associated NTK and the convergence to the global optima with a linear rate can be proved.
Later, the NTK analysis of global convergence is extended to multi-layer neural nets~\cite{AllenZhu2018LearningAG, Zou2018StochasticGD, AllenZhu2018ACT,mohamadi2023fast,wang2023ntk}. 

An alternative research direction attempts to examine the infinite-width neural network from a mean-field perspective~\citep{Sirignano2018MeanFA,Mei2018AMF, Chizat2018OnTG, nguyen2023rigorous}. The key idea is to characterize the learning dynamics of noisy SGD as the gradient flows over the space of probability distributions of neural network parameters. When the time goes to infinity, the noisy SGD converges to the unique global optimum of the convex objective function for the two-layer neural network with continuous width. 

NTK and mean-field methods provide a fundamental understanding on training dynamics and convergence properties of non-convex optimization in infinite-width neural networks. 
However, the assumption of infinitely wide neural networks does not hold in real-world applications, making it an open question whether NTK and mean field theories can be applied to the analysis of practical neural networks. Some recent research findings \cite{DBLP:conf/msml/SeleznovaK21, vyas2022limitations} indicate that the NTK theory does not generally describe the training dynamics of finite-width DNNs accurately and suggest that an entirely new conceptual viewpoint is required to provide a full theoretical analysis of DNNs' behavior under GD.

Therefore, elucidating the optimization and over-parameterization mechanisms underlying deep learning remains an unresolved issue at present~\cite{Oneto2023DoWR}.

\section{Preliminaries}
\label{sec:pre}
\subsection{Notation}
\begin{enumerate}
    \item Random variables are denoted using upper case letters such as $Z$, $X$, and $Y$, which take values in sets $\mathcal{Z}$, $\mathcal{X}$ and $\mathcal{Y}$, respectively. $\mathcal{X}$ represents the input feature space and $\mathcal{Y}$ denotes a finite set of labels. The cardinality of set $\mathcal{Z}$ is denoted by $|\mathcal{Z}|$. 
    \item We utilize the notation $f_\theta$ (hereinafter abbreviated as $f$) to denote the model characterized by the parameter vector $\theta$. Additionally, $f_{\theta}(x)$ (abbreviated as $f(x)$) represents the model with a parameter vector $\theta$ and a specific input $x$. The space of models, which is a set of functions endowed with some structure, is represented by $\mathcal{F}_{\Theta}=\{f_\theta:\theta\in \Theta\}$, where $\Theta$ denotes the parameter space. Similarly, we define the hypothesis space with feature $x$ as $\mathcal{F}_{\Theta}(x):=\{f_{\theta}(x):\theta\in \Theta\}$.
    \item The Legendre-Fenchel conjugate of a function $\Omega$ is denoted by $\Omega^*(\nu):= \sup_{\mu \in \mathrm{dom}(\Omega)}\langle \mu,\
    \nu \rangle-\Omega(\mu)$~\cite{Todd2003ConvexAA}. By default, $\Omega$ is a continuous strictly convex function, and its gradient with respect to $\mu$ is denoted by $\nabla_\mu \Omega(\mu)$. For convenience, we use $\mu_\Omega^*$ to represent $\nabla_\mu \Omega(\mu)$. When $\Omega(\cdot)=\frac{1}{2}\|\cdot\|_2^2$, we have $\mu=\mu _\Omega^*$. The Fenchel-Young loss $d_\Omega \colon \mathrm{dom}(\Omega) \times \mathrm{dom}(\Omega^*) \to \mathbb{R}_{\ge 0}$ \label{def:FY_loss} generated by $\Omega$ is defined as~\citep{Blondel2019LearningWF}:
\begin{equation}
d_{\Omega}(\mu, \nu) 
:= \Omega(\mu) + \Omega^*(\nu) - \langle \mu,\nu\rangle,
\label{eq:fy_losses}
\end{equation}
where $\Omega^*$ denotes the conjugate of $\Omega$.
\item The maximum and minimum eigenvalues of matrix $A$ are denoted by $\lambda_{\max}(A)$ and $\lambda_{\min}(A)$, respectively.
\item To improve clarity and conciseness, we transform the distribution function into a vector for processing and provide the following definitions of symbols:
\begin{equation}
    \begin{aligned}
        &q_{\mathcal{X}}:=(q_{X}(x_{1}),\cdots, q_{X}(x_{|\mathcal X|}))^\top,\\
        &q_{\mathcal{Y}|x}:=(q_{Y|X}(y_{1}|x),\cdots,q_{Y|X}(y_{|\mathcal Y|}|x))^\top.
    \end{aligned}
\end{equation}
Here, $q_{X}(x)$ and $q_{Y|X}(y|x)$ represent the marginal and conditional PMFs/PDFs, respectively. 
\end{enumerate}

\subsection{Lemmas}
\label{appendix:lem}

The theorems in this paper rely on the following lemmas. 
 
\begin{lemma}[Properties of Fenchel-Young losses]
\label{prop:fy_losses}
The following are the properties of Fenchel-Young losses~\citep{Blondel2019LearningWF}.
\begin{enumerate}
\item $d_{\Omega}(\mu, \nu) \ge 0$ for any $\mu \in
    \mathrm{dom}(\Omega)$ and $\nu \in \mathrm{dom}(\Omega^*)$. If $\Omega$ is a lower semi-continuous proper convex function, then the loss is zero iff $\nu \in \partial \Omega(\mu)$. Furthermore, when $\Omega$ is strictly convex, the loss is zero iff $\nu=\nabla_\mu \Omega(\mu)$.
\item If $\Omega$ is strictly convex, then $d_\Omega(\mu,\nu)$ is differentiable and  $\nabla_\nu d_{\Omega}(\mu,\nu) = \nabla_{\nu} \Omega^*(\nu)-\mu=\nu_{\Omega^*}^*-\mu$. If $\Omega$ is strongly convex, then $d_\Omega(\mu,\nu)$ is smooth, i.e., $\nabla_\nu d_\Omega(\mu,\nu)$ is Lipschitz continuous. 
\end{enumerate}
\end{lemma}
Throughout this paper, in line with practical applications and without affecting the conclusions, we assume that $\Omega$ is a proper, lower semi-continuous (l.s.c.), and strictly convex function or functional. Consequently, the Fenchel-Young loss attains zero if and only if $\mu =\nabla_\nu \Omega^*(\nu)$.
\begin{lemma}
\label{prop:high_dim_2}
    Suppose that we sample $n$ points $x^{(1)}, \ldots, x^{(n)}$ uniformly from the unit ball $\mathcal{B}^{m}_1:\{x\in\mathbb R^m,\|x\|_2\le 1\}$. Then with probability $1-O(1 / n)$ the following holds~\citep[Theorem 2.8]{blum2020foundations}:
\begin{equation}
    \begin{aligned}
                &\left\|x^{(i)}\right\|_2 \geq 1-\frac{2 \log n}{m},\text{ for } i=1, \ldots, n.\\
        & |\langle x^{(i)}, x^{(j)}\rangle| \leq \frac{\sqrt{6 \log n}}{\sqrt{m-1}}  \text { for all } i, j=1, \ldots, i \neq j.   
    \end{aligned}
\end{equation}
\end{lemma}
\begin{lemma}[Gershgorin’s circle theorem]
\label{lemma:gershgorin}
    Let $A$ be a real symmetric $n\times n$ matrix, with entries $a_{ij}$. For $i=\{1,\cdots,n\}$ let $R_i$ be the sum of the absolute value of the non-diagonal entries in the $i$-th row: $R_i=\sum_{1\le j\le n,j\neq i}|a_{ij}|$.
    For any eigenvalue $\lambda$ of $A$, there exists $i\in \{1,\cdots,n\}$ such that \begin{equation}
        |\lambda-a_{ii}|\le R_i.
    \end{equation}
\end{lemma}

\subsection{Setting and Basics}
The random pair $Z=(X,Y)$ follows the distribution $q$ and takes values in the product space $\mathcal{Z}=\mathcal{X}\times \mathcal{Y}$. The sample or training data, denoted by an $n$-tuple $s^n :=\{z^{(i)}\}_{i=1}^n=\{(x^{(i)},y^{(i)})\}_{i=1}^n$, is composed of i.i.d. samples drawn from the unknown true distribution $q_{\mathcal{Z}}$ (abbreviated as $q$).  
let $\tilde{q}$ denote an unbiased and consistent estimator of $q$. In the absence of prior knowledge about $q$, the method estimates $q$ based on event frequencies, denoted by $\hat{q}(s^n) := \frac{1}{n} \sum_{i=1}^n \mathbf{1}_{\{z\}}(z^{(i)})$, which is a commonly utilized approach. 
In this article, we refer to $\hat{q}$ as the frequency distribution estimator to distinguish it from the general estimator $\tilde{q}$. 
Real-world distributions are usually smooth, which inspires the common practice of smoothing $\hat{q}$ to construct $\tilde{q}$ to approximate the true distribution $q$ in machine learning and parameter estimation.
These smoothing techniques effectively integrate prior knowledge and are often advantageous, as exemplified by classical Laplacian smoothing. For simplicity, when the sample dataset is given, we abbreviate the results estimated by the estimator $\tilde{q}(s^n)$ and $\hat{q}(s^n)$ as $\tilde{q}$ and $\hat{q}$, respectively.

\section{Main Results}\label{sec:gd_framework}
This section introduces the methods and theoretical conclusions of this paper and demonstrates their application in understanding the optimization mechanisms of deep learning.

\subsection{Problem Formulation}\label{subsec:problem_def}

Almost all commonly used loss functions in statistics and machine learning, including cross-entropy, squared loss, and the Perceptron loss, fall under the category of Fenchel-Young losses $d_\Omega(y,f_{\theta}(x))$ by designing different $\Omega$~\citep{Blondel2019LearningWF}. Therefore, to ensure the broad applicability of the findings presented herein and their consistency with practical applications, the Fenchel-Young loss is adopted as the default loss function in this paper. Therefore, the expected risk $\mathcal{R}(f_\theta)$ and empirical risk $\mathcal{R}_n(f_\theta)$ of a hypothesis $ f_\theta $ under a given loss function $d_{\Omega}(\cdot,\cdot)$ are expressed as follows:
\begin{equation}
    \begin{aligned}
    \mathcal{R}(f_\theta)&=\mathbb{E}_{(X,Y)\sim q} \, d_\Omega(\mathbf{1}_{\{y\}}, f_\theta(x))\\
        \mathcal{R}_n(f_\theta)&=\mathbb{E}_{(X,Y)\sim \tilde{q}} \, d_\Omega(\mathbf{1}_{\{y\}}, f_\theta(x)).
    \end{aligned}
\end{equation}
Here, $\mathbf{1}_{\{y\}}$ denotes the indicator vector of $y$, with a dimension of $|\mathcal{Y}|$. It is a one-hot vector where the entry corresponding to $y$ is 1, and all other entries are 0.
To briefly illustrate the generality of the aforementioned expected risk and empirical risk, consider an example where $\Omega(p) = \sum_{i=1}^{|p|} p_i\log p_i - p_i$. In this case, we obtain:

\begin{equation}\label{eq:softmax}
    d_\Omega(\mathbf{1}_{\{y\}},f_\theta(x)) = D_{KL}(\mathbf{1}_{\{y\}}, p),
\end{equation}

where $p_i = e^{(f_\theta(x))_i}/\sum_{i=1}^{|f_\theta(x)|}e^{(f_\theta(x))_i}$. Here, $(f_\theta(x))_i$ denotes the $i$-th element of $f_\theta(x)$. This formulation corresponds to the widely used softmax cross-entropy loss in classification scenarios. 

The proposition below illustrates that optimizing the expected risk and empirical risk is equivalent to enabling the model to learn the conditional distribution of labels given features. 
\begin{proposition}
    \label{prop:decompose}
Let $B_\Omega(q_{\mathcal{XY}}) = \mathbb{E}_{Y\sim q_{\mathcal{Y}}}[\Omega(\mathbf{1}_{\{y\}})] - \mathbb{E}_{X\sim q_{\mathcal{X}}}[\Omega(q_{\mathcal{Y}|x})]$, then we have:
\begin{equation}
\begin{aligned}
\mathcal{R}(f_\theta) &= B_\Omega(q_{\mathcal{XY}}) + \mathbb{E}_{X\sim q_{\mathcal{X}}}d_\Omega(q_{\mathcal{Y}|x}, f_\theta(x)) \\
\mathcal{R}_n(f_\theta) &= B_\Omega(\tilde{q}_{\mathcal{XY}}) + \mathbb{E}_{X\sim \tilde{q}_{\mathcal{X}}}d_\Omega(\tilde{q}_{\mathcal{Y}|x}, f_\theta(x))
\end{aligned}
\end{equation}
\end{proposition}
\begin{proof}
Based on the definition of the Fenchel-Young loss, we have:
\begin{equation}
    \begin{aligned}
        \mathcal{R}_n(f_\theta) &= \mathbb{E}_{X \sim \tilde{q}_{\mathcal{X}}} \mathbb{E}_{Y \sim \tilde{q}_{\mathcal{Y}|x}} \big[ \Omega(\mathbf{1}_{\{y\}}) + \Omega^*(f_\theta(x)) - \mathbf{1}_{\{y\}}^\top f_\theta(x) \big] \\
        &= \mathbb{E}_{Y \sim \tilde{q}_{\mathcal{Y}}} \Omega(\mathbf{1}_{\{y\}}) + \mathbb{E}_{X \sim \tilde{q}_{\mathcal{X}}} \big[ \Omega^*(f_\theta(x)) - \tilde{q}_{\mathcal{Y}|x}^\top f_\theta(x) \big] \\
        &= \mathbb{E}_{Y \sim \tilde{q}_{\mathcal{Y}}} \Omega(\mathbf{1}_{\{y\}}) + \mathbb{E}_{X \sim \tilde{q}_{\mathcal{X}}} \big[ \Omega^*(f_\theta(x)) + \Omega(\tilde{q}_{\mathcal{Y}|x}) - \tilde{q}_{\mathcal{Y}|x}^\top f_\theta(x) - \Omega(\tilde{q}_{\mathcal{Y}|x}) \big] \\
        &= \mathbb{E}_{Y \sim \tilde{q}_{\mathcal{Y}}} \Omega(\mathbf{1}_{\{y\}}) - \mathbb{E}_{X \sim \tilde{q}_{\mathcal{X}}} \big[ \Omega(\tilde{q}_{\mathcal{Y}|x}) \big] + \mathbb{E}_{X \sim \tilde{q}_{\mathcal{X}}} \big[ d_\Omega(\tilde{q}_{\mathcal{Y}|x}, f_\theta(x)) \big].
    \end{aligned}
\end{equation}
Similarly, we can derive:
\begin{equation}
    \begin{aligned}
        \mathcal{R}(f_\theta) = \mathbb{E}_{Y \sim q_{\mathcal{Y}}} \Omega(\mathbf{1}_{\{y\}}) - \mathbb{E}_{X \sim q_{\mathcal{X}}} \big[ \Omega(q_{\mathcal{Y}|x}) \big] + \mathbb{E}_{X \sim q_{\mathcal{X}}} \big[ d_\Omega(q_{\mathcal{Y}|x}, f_\theta(x)) \big].
    \end{aligned}
\end{equation}
\end{proof}

Given that both $B_\Omega(q_{\mathcal{XY}})$ and $B_\Omega(\tilde{q}_{\mathcal{XY}})$ are terms independent of the model $f_\theta$, it follows that solving a classification problem is essentially about learning the conditional distribution of labels given the features. Therefore, a substantial proportion of machine learning tasks can indeed be viewed as estimating the underlying probability distribution. Indeed, this proposition provides the theoretical foundation for analyzing the optimization mechanisms of deep learning from the perspective of probability distribution estimation.
To clarify, the Fenchel-Young loss measures the discrepancy between the true conditional distribution $q_{\mathcal{Y}|x}$ and the transformed predictions $(f_\theta(x))_{\Omega^*}^*$, rather than directly comparing $q_{\mathcal{Y}|x}$ and $f_\theta(x)$. This subtle yet important property underpins the theoretical foundation and practical effectiveness of the Fenchel-Young loss in various machine learning tasks.

It is important to note that while this paper adopts the Fenchel-Young loss as the loss function, it does not explicitly consider the influence of regularization terms. The specific reason for this omission is articulated through the following theorem, which demonstrates that the Fenchel-Young loss inherently possesses a parameter regularization capability.
\begin{theorem}
\label{thm:native_opt_exp}
Let  $p_{\mathcal{Y}|x}=(f_\theta(x))_{\Omega^*}^*$. Given that when $\theta = \mathbf{0}$, $\forall x \in \mathcal{X}$, $p_{\mathcal{Y}|x}$ is the uniform distribution over $|\mathcal{Y}|$, it follows that:
1. There exist constants $k > 0$ and $\epsilon > 0$ such that
\begin{equation}
    \|p_{\mathcal{Y}|x}\|_2^2 \leq k\|\theta\|_2^2 + 1/|\mathcal{Y}|,
\end{equation}
2. The optimization problem $\min_{\theta \in \Theta'_\epsilon} \mathcal{R}(q,p)$ is equivalent to
\begin{equation}
    \min_{\theta \in \Theta'_\epsilon} \left\{ \mathbb{E}_{XY\sim \tilde{q}_{\mathcal{XY}}} \langle \mathbf{1}_{y}, -(f_{\theta}(x))_{\Omega^*}^* \rangle + k \|\theta\|^2_2 \right\},
\end{equation}
where $\Theta'_\epsilon = \{\theta \mid \nabla^2_{\theta}\|(f_{\theta}(x))_{\Omega^*}^*\|_2^2 \succeq 0 \text{ and } \|\theta\|_2^2 \leq \epsilon\}$.
\end{theorem}
\begin{proof}
    Since $\theta = \mathbf{0}$ represents the minimum point of $\|p_{\mathcal{Y}|x}\|^2_2$, where $p_{\mathcal{Y}|x} = (f_{\theta}(x))_{\Omega^*}^*$, it follows that $\forall \theta \in \Theta'$:
\begin{equation}
    \begin{aligned}
\nabla_{\theta}\|(f_{\mathbf{0}}(x))_{\Omega^*}^*\|_2^2 = 0,\\
\nabla^2_{\theta}\|(f_{\theta}(x))_{\Omega^*}^*\|_2^2 \succeq 0.
    \end{aligned}
\end{equation}
By applying the mean value theorem, there exists $\theta' = \alpha \mathbf{0} + (1 - \alpha)\theta, \alpha \in [0, 1]$ such that the following inequalities hold:
\begin{equation}
    \begin{aligned}
       \|(f_{\theta}(x))_{\Omega^*}^*\|_2^2 &= \|(f_{\mathbf{0}}(x))_{\Omega^*}^*\|_2^2+\theta^\top 
\nabla^2_{\theta}\|(f_{\theta'}(x))_{\Omega^*}^*\|_2^2\theta,\\
        &\le 1/|\mathcal{Y}|+\|\theta\|_2\|
\nabla^2_{\theta}\|(f_{\theta'}(x))_{\Omega^*}^*\|_2^2\theta\|_2,\\
        &\le 1/|\mathcal{Y}|+\|\theta\|_2^2 \|
\nabla^2_{\theta}\|(f_{\theta'}(x))_{\Omega^*}^*\|_2^2\|_2,
        \end{aligned}
\end{equation}
where the first inequality is based on the Cauchy-Schwarz inequality, and the second inequality follows from the definition of the matrix-induced 2-norm.
Consequently, within the neighborhood of $\theta=\mathbf{0}$, that is $\Theta'$, we have
\begin{equation}
     1/|\mathcal{Y}|\le \|(f_{\theta}(x))_{\Omega^*}^*\|_2^2\leq 1/|\mathcal{Y}|+ k \|\theta\|_2^2,
\end{equation}
where $k\ge  \|
\nabla^2_{\theta}\|(f_{\theta'}(x))_{\Omega^*}^*\|_2$.
That is, there exists $k > 0$ such that
\begin{equation}
    \|p_{\mathcal{Y}|x}\|_2^2  \leq k\|\theta\|_2^2+1/|\mathcal{Y}|.
\end{equation}
It indicates that $\min_\theta \|\theta\|_2^2$ and $\min_\theta \|p_{\mathcal{Y}|x}\|_2^2$ are equivalent.
Therefore, $\min_\theta \mathbb{E}_X \|p_{\mathcal{Y}|X}\|_2^2$ is equivalent to minimizing $\min_\theta {\|\theta\|_2^2}$.
Given the equation:
$$
\begin{aligned}
 \mathbb{E}_X[ \mathcal{E}_f(q_{\mathcal{Y}|x},p_{\mathcal{Y}|x})= \mathbb{E}_X[ \|\tilde{q}_{\mathcal{Y}|X}\|_2^2+\|p_{\mathcal{Y}|X}\|_2^2-  2\langle  q_{\mathcal{Y}|X}, p_{\mathcal{Y}|X}\rangle ]
\end{aligned}
$$
and since $\mathbb{E}_X[ \|q_{\mathcal{Y}|X}\|_2^2]$ is model-independent, it follows that $\min_\theta \mathbb{E}_{X\sim q_{\mathcal{X}}}[ \mathcal{E}_f(q_{\mathcal{Y}|x},p_{\mathcal{Y}|x})$ is equivalent to
$$
\min_\theta \{\mathbb{E}_{XY\sim q_{\mathcal{XY}}}  \langle \mathbf{1}_{y}, -(f_{\theta}(x))_{\Omega^*}^* \rangle+\lambda \|\theta\|_2^2\},
$$
where $\theta\in \Theta'$.
\end{proof}
This theorem reveals that the Fenchel-Young loss inherently possesses an implicit regularization capability. The condition $f_{\mathbf{0}}(x) = \mathbf{0}$ holds true for the vast majority of DNNs, thereby enabling $\|\theta\|_2^2$ to control the uniformity of the model's output distribution. This provides a novel perspective on why regularization terms can mitigate overfitting: by constraining the norm of the parameters ($\|\theta\|_2^2$), the model is encouraged to produce more uniform predictions, reducing the risk of overfitting the training data. 

To articulate our approach, we list the definitions used in proposed method as follows.
\begin{itemize}
\item \textbf{Distribution Fitting Error. }This error measures the performance of the model in fitting the target $q$, denoted by $\mathcal{E}_f(q,p)=\|p-q\|_2^2$, where $p$ represents the predicted distribution produced by the model. 
\item \textbf{Gradient Norm.} In this paper, we refer to $\|\nabla_\theta d_\Omega(q_{\mathcal{Y}|
            x},f_\theta(x))\|_2^2$ as the \textbf{gradient norm}.  
\item \textbf{Structural Matrix.} We define $ A_x := \nabla_\theta f_\theta(x)^\top \nabla_\theta f_\theta(x)$ as the structural matrices corresponding to the model with input $x$.
\item  \textbf{Structural Error.} We define the structural error of $f_\theta(x)$ with structural matrix $A_x$ as follows:
    \begin{equation}
    S(\alpha, \beta, \gamma, A_x) = \alpha G(A_x) + \beta U(A_x) + \gamma L(A_x),
    \end{equation}
    where
\begin{equation}
        \begin{aligned}
    U(A_x) &= -\log \lambda_{\min}(A_x), \\
    L(A_x) &= -\log \lambda_{\max}(A_x), \\
    G(A_x) &= U(A_x) - L(A_x).
    \end{aligned}
\end{equation}
    Here, $\alpha$, $\beta$, and $\gamma$ are positive real numbers representing the weights associated with $G(A_x)$, $U(A_x)$, and $L(A_x)$, respectively. 
\end{itemize}

\subsection{Non-conversion Optimization Mechanism}\label{subsec:logic}

We provide the upper and lower bounds of the distribution fitting error under the condition of using a Fenchel-Young loss as the loss function, as follows:

\begin{theorem}\label{thm:station_error}
If $\lambda_{\min}(A_x) \neq 0$ for all $x \in \mathcal{X}$, we have
\begin{equation}
\begin{aligned}
       L(A_x) \leq \log \frac{\mathcal{E}_f(\tilde{q}_{\mathcal{Y}|x}, p_{\mathcal{Y}|x})}{\|\nabla_\theta d_\Omega(\tilde{q}_{\mathcal{Y}|x}, f_\theta(x))\|_2^2} \leq U(A_x),
\end{aligned}
\end{equation}
where $p_{\mathcal{Y}|x} = (f_\theta(x))_{\Omega^*}^*$.
\end{theorem}
\begin{proof}
According to the given definition, we have
\begin{equation}
            \|\nabla_\theta d_\Omega(\tilde{q}_{\mathcal{Y}|x},f_\theta(x))\|_2^2=e^\top A_x e,
\end{equation}
where $e=\tilde{q}_{\mathcal{Y}|x}-p_{\mathcal{Y}|x}$, $p_{\mathcal{Y}|x}=(f_\theta(x))_{\Omega^*}^*$, $A_x=g_x^\top g_x$ is the structure matrix ,  $g_x= \nabla_\theta f_\theta(x)$.
Let $A_x=UDU^\top$ denote the eigenvalue decomposition of $A_x$. 
Consequently, $D$ is a diagonal matrix whose entries are the eigenvalues of $A_x$.
We then have
\begin{equation}
    \begin{aligned}
        \min_{\|e\|_2=K} e^\top A_x e&=\min_{\|e\|_2^2=K} e^\top (UDU^\top)e\\
        &=\min_{\|e\|_2^2=K} (U^\top e)^\top D (U^\top e)\\
        &=\min_{\|U^\top e\|_2^2=K} (U^\top e)^\top D (U^\top e)\\
        &=K\min_{x\in \mathcal{S}}\sum_i x_i D_{ii},
    \end{aligned}
\end{equation}
where $\mathcal{S}=\{(x_1,\dots,x_N)\in\mathbb{R}^N\mid x_i\geq 0,~~\sum_{i}x_i=1\}$.
The final step is equivalent to 
\begin{equation}
    K\min_{x\in\mathcal{S}}\sum_{i}x_iD_{ii}=K\min_{i}D_{ii}=\lambda_{\min}(A_x)\|e\|_2^2.
\end{equation}
Similarly, we obtain 
\begin{equation}
    \max_{\|e\|_2=K} e^\top A_x e=\lambda_{\max}(A_x)\|e\|_2^2.
\end{equation}
It follows that 
\begin{equation}
   \lambda_{\min}(A_x)\mathcal{E}_f(\tilde{q}_{\mathcal{Y}|x},p_{\mathcal{Y}|x})\le  \|\nabla_\theta d_\Omega(\tilde{q},f)\|_2^2\le \lambda_{\max}(A_x)\mathcal{E}_f(\tilde{q}_{\mathcal{Y}|x},p_{\mathcal{Y}|x}),
\end{equation}
where $p_{\mathcal{Y}|x}=(f_\theta(x))_{\Omega^*}^*$.
\end{proof}
Theorem~\ref{thm:station_error} establishes the equivalence between the distribution fitting error $\mathbb{E}_X\mathcal{E}_f(\tilde{q}_{\mathcal{Y}|x}, p_{\mathcal{Y}|x})$ and the gradient norm $\mathbb{E}_X \|\nabla_\theta d_\Omega(\tilde{q}_{\mathcal{Y}|x}, f_\theta(x))\|_2^2$. Consequently, this theorem indicates that the non-convex objective can be optimized by reducing the gradient norm and the structural error $\mathbb{E}_X S(\alpha, \beta, \gamma, A_x)$ to manage its upper and lower bounds. Under the reasonable assumption that $d_\Omega(q_{\mathcal{Y}|x}, f_\theta(x))$ is Lipschitz smooth with respect to $\theta$, the application of SGD guarantees the convergence of $\|\nabla_\theta d_\Omega(\tilde{q}_{\mathcal{Y}|x}, f_\theta(x))\|_2^2$. Therefore, SGD algorithms effectively minimize $\mathbb{E}_X \|\nabla_\theta d_\Omega(\tilde{q}_{\mathcal{Y}|x}, f_\theta(x))\|_2^2$, thereby driving the optimization process toward reducing the distribution fitting error $\mathbb{E}_X \mathcal{E}_f(\tilde{q}_{\mathcal{Y}|x}, p_{\mathcal{Y}|x})$ when $S(\alpha, \beta, \gamma, A_x),\forall x\in \mathcal{X}$ does not increase.

\subsection{Structural Error Optimization}
\label{subsec:structural_error}
SGD algorithms ensure the minimization of the gradient norm. In this subsection, we analyze another critical factor in the non-convex optimization mechanism of deep learning: structural error. We demonstrate that the model's architecture, the number of parameters, and the independence among parameter gradients can all be leveraged to reduce structural error.

\subsubsection{Skip Connection}

During the training of neural networks, the magnitudes of the elements in the gradient vector $\nabla_\theta f_\theta(x)$ diminish alongside the reduction in backpropagated errors. According to Theorem~\ref{thm:station_error}, under conditions where the gradient norm is held constant, the eigenvalues of the structural matrix—corresponding to the singular values of $\nabla_\theta f_\theta(x)$—dictate the extent of distribution fitting error; specifically, larger eigenvalues correlate with smaller distribution fitting errors. Consequently, one direct approach to mitigating structural errors lies in augmenting the eigenvalues of the structural matrix through modifications in network architecture.
From this perspective, Residual blocks~\cite{He2015DeepRL} represent a classic and highly effective architecture. By introducing skip connections, they are able to effectively enhance the eigenvalues of the structural matrix.

Here, we present a concise mathematical explanation. Following the random initialization of model parameters or upon reaching a stable convergence state, the gradients of the model $f_\theta(x)$ with respect to its parameters often tend toward zero. Suppose that $f_\theta$ is composed of $k$ sequential blocks, where the output of each block is represented by a function $h^i$, and the final output satisfies $h^{(k)} = f_\theta(x)$. For these $k$ blocks, we introduce skip connections, thereby constructing a new model $g_\theta(x)$.  
Using the chain rule, we have:
\begin{equation}
\begin{aligned}
    \nabla_{\theta^{(j)}} f_\theta(x) &= \nabla_{h^{(k-1)}} f_\theta(x) \nabla_{h^{(k-2)}} h^{(k-1)} \cdots \nabla_{\theta^{(j)}} h^{(j)} \\
    &= \left( \prod_{i=j}^{k-1} \nabla_{h^{(i)}} h^{(i+1)} \right) \nabla_{\theta^{(j)}} h^{(j)}, \\
    \nabla_{\theta^{(j)}} g_\theta(x) &= \left( \nabla_{h^{(k-1)}} f_\theta(x) + I \right) \left( \nabla_{h^{(k-2)}} h^{(k-1)} + I \right) \cdots \nabla_{\theta^j} h^{(j)} \\
    &= \left( \prod_{i=j}^{k-1} \left( \nabla_{h^{(i)}} h^{(i+1)} + I \right) \right) \nabla_{\theta^{(j)}} h^{(j)},
\end{aligned}
\end{equation}
where $\theta^{(j)}$ represents the parameters corresponding to the $j$-th block, and $I$ is the identity matrix. It is evident that when the elements of $\nabla_{h^{(i)}} h^{(i+1)}$ are close to zero, the elements of $\nabla_{\theta^j} f_\theta(x)$ also approach zero, causing the eigenvalues of the structural matrix to become close to zero as well. However, with the introduction of skip connections, $\nabla_{\theta^{(j)}} g_\theta(x) \approx \nabla_{\theta^{(j)}} h^{(j)}$, which prevents the decay caused by the multiplication of gradients in the chain rule. Therefore, skip connections prevent the eigenvalues of the structural matrix from becoming excessively small due to gradient descent, which would otherwise lead to significant structural errors. By maintaining larger eigenvalues, skip connections facilitate the optimization of the distribution fitting error by reduction of the gradient norm.

\subsubsection{Parameter Number and Independence}
The number of parameters in a model, along with the independence of these parameters, can also be leveraged to reduce structural error. Our analysis is grounded in the following condition:
\begin{definition}[Gradient Independence Condition]
\label{def:grad_indep_con}
Each column of $\nabla_\theta f(x)$ is approximated as uniformly sampled from a ball $\mathcal{B}^{|\theta|}_\epsilon: \{x \in \mathbb{R}^{|\theta|}, \|x\|_2 \le \epsilon\}$.
\end{definition}
This condition essentially posits that the derivatives of the model's output with respect to its parameters are finite and approximately independent. 
In practical scenarios, gradients computed during the training of neural networks are indeed finite due to mechanisms such as gradient clipping and the inherent properties of activation functions used (e.g., ReLU, sigmoid). Furthermore, under certain conditions, such as when using techniques like dropout, it is reasonable to assume that the dependence among parameter gradients is weakened. 
In the experimental section~\ref{sec:experiment}, we observe that under conditions of large-scale parameter counts, the Gradient Independence Condition~\ref{def:grad_indep_con} is generally satisfied immediately after random initialization of the parameters. 
However, as the training process begins, backpropagated errors are relatively large, leading to significant changes in the parameters. During this phase, our experimental results suggest that the this condition  may not hold. 
As training continues and the model enters a more stable phase where the magnitude of backpropagated errors decreases, experimental results imply that the condition becomes approximately valid again.
Therefore, this condition serves as a simplification of the model's characteristics, which is reasonable and largely consistent with practical scenarios under certain circumstances. 

Below, we will derive a theorem based on this condition to describe how the number of parameters and inter-parameter correlations affect the structural error. 
\begin{theorem}
\label{thm:number_para}
If the model satisfies the Gradient Independence Condition~\ref{def:grad_indep_con} and $\tilde{q}_{\mathcal{Y}|x}\le |\theta|-1$, then with probability at least $1-O(1 / |\tilde{q}_{\mathcal{Y}|x}|)$ the following holds:
\begin{equation}
    \begin{aligned}
    L(A_x)\leq U(A_x)&\le -\log (1-\frac{2 \log |\tilde{q}_{\mathcal{Y}|x}|}{|\theta|})^2-\log \epsilon^2,\\
      G(A_x)=U(A_x&)-L(A_x)\le \log (Z(|\theta|,|\tilde{q}_{\mathcal{Y}|x}|)+1),
    \end{aligned}
\end{equation}
where $Z(|\theta|,|\tilde{q}_{\mathcal{Y}|x}|)=\frac{2|\tilde{q}_{\mathcal{Y}|x}|\sqrt{6 \log |\tilde{q}_{\mathcal{Y}|x}|}}{\sqrt{|\theta|-1}} (1-\frac{2 \log |\tilde{q}_{\mathcal{Y}|x}|}{|\theta|})^2$, is a decreasing function of $|\theta|$ and an increasing function of $|\tilde{q}_{\mathcal{Y}|x}|$.
\end{theorem}
\begin{proof}
    Since each column of $\nabla_\theta f(x)$ is uniformly from the ball $\mathcal{B}^{|\theta|}_\epsilon:\{x\in\mathbb R^{|\theta|},\|x\|_2\le \epsilon\}$, according to Lemma~\ref{prop:high_dim_2}, then with probability $1-O(1 / |\tilde{q}_{\mathcal{Y}|x}|)$ the following holds:
\begin{equation}\label{eq:bound_elememnts}
    \begin{aligned}
            &(A_x)_{ii}=\left\|(\nabla_\theta f(x))_i \right\|^2_2 \geq \left(1-\frac{2 \log |\tilde{q}_{\mathcal{Y}|x}|}{|\theta|}\right)^2\epsilon^2,\text{ for } i=1, \ldots, |h|,\\
            &|(A_x)_{ij}|=|\langle (\nabla_\theta f(x))_i, (\nabla_\theta f(x))_j\rangle| \leq \frac{\sqrt{6 \log |\tilde{q}_{\mathcal{Y}|x}|}}{\sqrt{|\theta|-1}} \epsilon^2\quad \text { for all } i, j=1, \ldots,|h|, i \neq j.
    \end{aligned}
    \end{equation}
Because $A_x=\nabla_\theta f(x)^\top \nabla_\theta f(x)$ is a symmetric positive semidefinite matrix, then we have $\lambda_{\min}(A_x)\le (A_x)_{ii}\le \lambda_{\max}(A_x)$.
Then the following holds with probability $1-O(1 / |\tilde{q}_{\mathcal{Y}|x}|)$ at least 
\begin{equation}\label{eq:ev_lower_bound}
\begin{aligned}
\lambda_{\min}(A_x)\geq (1-\frac{2 \log |\tilde{q}_{\mathcal{Y}|x}|}{|\theta|})^2\epsilon^2,\\
        \lambda_{\max}(A_x)\geq (1-\frac{2 \log |\tilde{q}_{\mathcal{Y}|x}|}{|\theta|})^2\epsilon^2.
\end{aligned}
\end{equation}

According to Lemma~\ref{lemma:gershgorin}, there exist $k,k'\in  \{1,\cdots,|\tilde{q}_{\mathcal{Y}|x}|\}$
\begin{equation}
    \begin{aligned}
        \lambda_{\max}(A_x)-(A_x)_{kk}&\le R_k,\\ 
        (A_x)_{k'k'}-\lambda_{\min}(A_x)&\le R_{k'},
    \end{aligned}
\end{equation}
where $R_k=\sum_{1\le j\le n,j\neq k}|(A_x)_{kj}|$, $R_{k'}=\sum_{1\le j\le n,j\neq k'}|(A_x)_{k'j}|$.
Based on the inequalities~\eqref{eq:bound_elememnts}, we obtain 
\begin{equation}\label{eq:bound_R}
    R_k+R_{k'}\le 2(|\tilde{q}_{\mathcal{Y}|x}|-1)\frac{\sqrt{6 \log |\tilde{q}_{\mathcal{Y}|x}|}}{\sqrt{|\theta|-1}}\epsilon^2.
\end{equation}

Since $(A_x)_{kk}\le \epsilon^2$ and $|\tilde{q}_{\mathcal{Y}|x}|\le |\theta|-1$, then we have
\begin{equation}
    \begin{aligned}
        \lambda_{\max}(A_x)-\lambda_{\min}(A_x)&\le R_k+R_{k'}+(A_x)_{kk}-(A_x)_{k'k'}\\
        &\le R_k+R_{k'}+\epsilon^2-(A_x)_{k'k'}\\
        &\le 2(|h|-1)\frac{\sqrt{6 \log |\tilde{q}_{\mathcal{Y}|x}|}}{\sqrt{|\theta|-1}}\epsilon^2+\epsilon^2-\left(1-\frac{2 \log |\tilde{q}_{\mathcal{Y}|x}|}{|\theta|}\right)^2\epsilon^2\\
        &\le 2(|\tilde{q}_{\mathcal{Y}|x}|-1)\epsilon^2\frac{\sqrt{6 \log |\tilde{q}_{\mathcal{Y}|x}|}}{\sqrt{|\theta|-1}}+\frac{4 \log |\tilde{q}_{\mathcal{Y}|x}|}{|\theta|}\epsilon^2\\
        &\le \frac{2|\tilde{q}_{\mathcal{Y}|x}|\epsilon^2\sqrt{6 \log |\tilde{q}_{\mathcal{Y}|x}|}}{\sqrt{|\theta|-1}}.
    \end{aligned}
\end{equation}
By dividing both sides of the aforementioned inequality by the corresponding sides of inequalities~\eqref{eq:ev_lower_bound}, we obtain:
\begin{equation}
    \begin{aligned}
   \frac{ \lambda_{\max}(A_x)}{\lambda_{\min}(A_x)}&\le Z(|\theta|,|\tilde{q}_{\mathcal{Y}|x}|)+1,
    \end{aligned}
\end{equation}
where $Z(|\theta|,|\tilde{q}_{\mathcal{Y}|x}|)=\frac{2|\tilde{q}_{\mathcal{Y}|x}|\sqrt{6 \log |\tilde{q}_{\mathcal{Y}|x}|}}{\sqrt{|\theta|-1}} (1-\frac{2 \log |\tilde{q}_{\mathcal{Y}|x}|}{|\theta|})^2$.
\end{proof}

According to this theorem, the strategies to reduce structural error involve two key aspects: enhancing gradient independence to satisfy the Gradient Independence Condition~\ref{def:grad_indep_con} and reducing the structural error by increasing the number of parameters $|\theta|$. 

Below, we analyze several deep learning techniques based on Theorem~\ref{thm:number_para}. 
Our examination includes insights derived from our designed experiments, which have validated several conclusions. It is important to note, however, that some of these conclusions are based on intuition and practical experience rather than rigorous theoretical derivation or comprehensive empirical validation. Despite this, we believe that presenting preliminary yet plausible analyses from a novel perspective can offer valuable insights and stimulate further discussion and research into these issues.

\paragraph{Enhancing Gradient Independence}
To achieve the Gradient Independence Condition, techniques that weaken inter-parameter correlations and promote independent gradient updates are essential. Some effective implementations include:
\begin{itemize}
    \item \textbf{Dropout}~\cite{Krizhevsky2012ImageNetCW}: By randomly deactivating neurons during training, dropout reduces co-adaptation among parameters, which in turn encourages more independent gradient updates. This process helps weaken the dependence among the columns of $\nabla_\theta f_\theta(x)$, thereby improving the rationality and applicability of the Gradient Independence Condition. 
    \item \textbf{Random parameter initialization}. Proper weight initialization methods (e.g., Xavier or He initialization) ensure that gradients are well-scaled and less correlated at the start of training, facilitating adherence to the Gradient Independence Condition.
    \item \textbf{Stochasticity in SGD}. Unlike gradient descent (GD), which computes gradients using the entire dataset and thus produces deterministic updates, SGD approximates the true gradient by computing it on a single sample or a small mini-batch randomly selected from the dataset. This introduces variability into the parameter updates at each iteration. From this perspective, SGD incorporates randomness through sampling, which helps to validate the Gradient Independence Condition.
\end{itemize}

\paragraph{Increasing the Number of Parameters}
The number of parameters in a model influences not only its fitting capability but also plays a crucial role in its non-convex optimization ability.
Theorem~\ref{thm:number_para} indicates that the structural error decreases as the number of parameters increases, provided that the Gradient Independence Condition is satisfied. This conclusion aligns with the pivotal viewpoint that over-parameterization plays a crucial role in the exceptional non-convex optimization and generalization capabilities of DNNs~\cite{Du2018GradientDP, Chizat2018OnLT, Arjevani2022AnnihilationOS}. Going further, this theorem provides theoretical insights into how over-parameterization influences a model's non-convex optimization capability. The experimental results in Figure~\ref{fig:lay_init_indicator} also validate this conclusion.

\section{Empirical Validations}
\label{sec:experiment}
In this section, we aim to validate the correctness of the core theoretical conclusions of this paper through experiments.  
For the sake of convenient description, we refer to $\log \mathcal{E}_f(\tilde{q}_{\mathcal{Y}|x},p_{\mathcal{Y}|x})$ as the fitting error, $\log \|\nabla_\theta d_\Omega(\tilde{q}_{\mathcal{Y}|x}, f_x)\|_2^2 + U(A_x)$ and $\log \|\nabla_\theta d_\Omega(\tilde{q}_{\mathcal{Y}|x}, f_x)\|_2^2 + L(A_x)$ as its upper and lower bounds, respectively. Additionally, we denote $\log_2 \|\nabla_\theta d_\Omega(\tilde{q}_{\mathcal{Y}|x}, f_x)\|_2^2$ simply as the gradient norm. 

\begin{figure}[ht]
\begin{center}
\centerline{\includegraphics[width=\columnwidth]{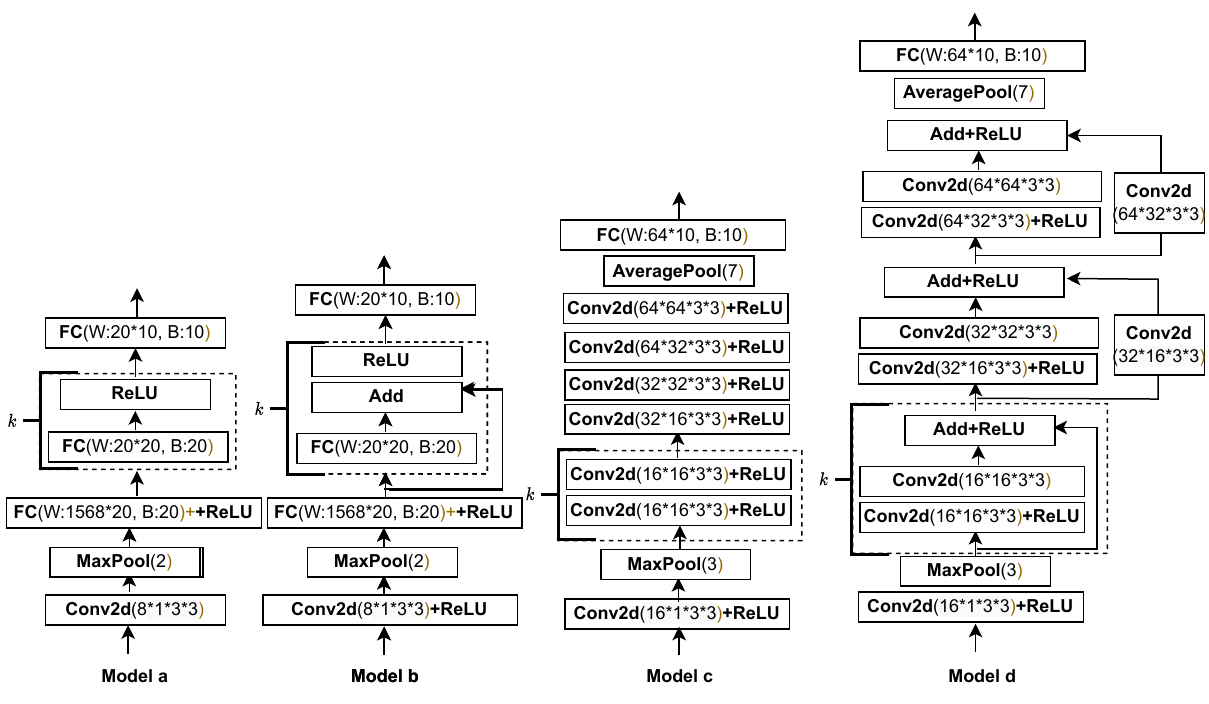}}
\caption{Model architectures and configuration parameters.}
\label{fig:model_increase}
\end{center}
\vskip -0.2in
\end{figure}
\subsection{Experimental Design}
Since the focus of this experiment is on evaluating whether models with different architectures align with the theoretical conclusions presented in this paper during the training process, rather than focusing on the final specific performance metrics or comparisons with existing algorithms, we use the \texttt{MNIST} dataset~\cite{LeCun1998GradientbasedLA} as a case study and design model architectures according to the conclusions to be verified. 

Model architectures and configuration parameters are illustrated in Figure~\ref{fig:model_increase}. Models b and d extend Models a and c by incorporating skip connections, with $k$ denoting the number of blocks (or layers), allowing for adjustable model depth.
The experiments were executed using the following computing environment: Python 3.7, Pytorch 2.2.2, and a GeForceRTX2080Ti GPU. 
The training parameters are as follows: the loss function is softmax cross-entropy, the number of epochs is 1, the batch size is 64, and the models were trained using the SGD algorithm with a learning rate of 0.01 and momentum of 0.9. 
In the experiments, we control the parameter number by adjusting the model depth.

\subsection{Experimental Result on Optimization Mechanism}
We assign a value of $k = 1$ to the models in Figure~\ref{fig:model_increase} and train them. 
The changes in gradient norm and structural error throughout the training process are illustrated in Figure~\ref{fig:convergence_indicator}. The fitting errors of the models, along with their respective bounds, are depicted in Figure~\ref{fig:convergence_bound}. To quantify the changes in correlations among these metrics as training progresses, we introduce the local Pearson correlation coefficient as a metric. Specifically, we apply a sliding window approach to compute the Pearson correlation coefficients between the given variables within each window. We use a sliding window of length 50 to compute the Pearson correlation coefficients between the fitting error and its upper and lower bounds. During training, the local Pearson correlation coefficients are illustrated in Figure~\ref{fig:model_correlation}.
\begin{figure}[ht]
	\centering
	\begin{minipage}{1.0\linewidth}
		\centering
            \centerline{\includegraphics[width=\columnwidth]{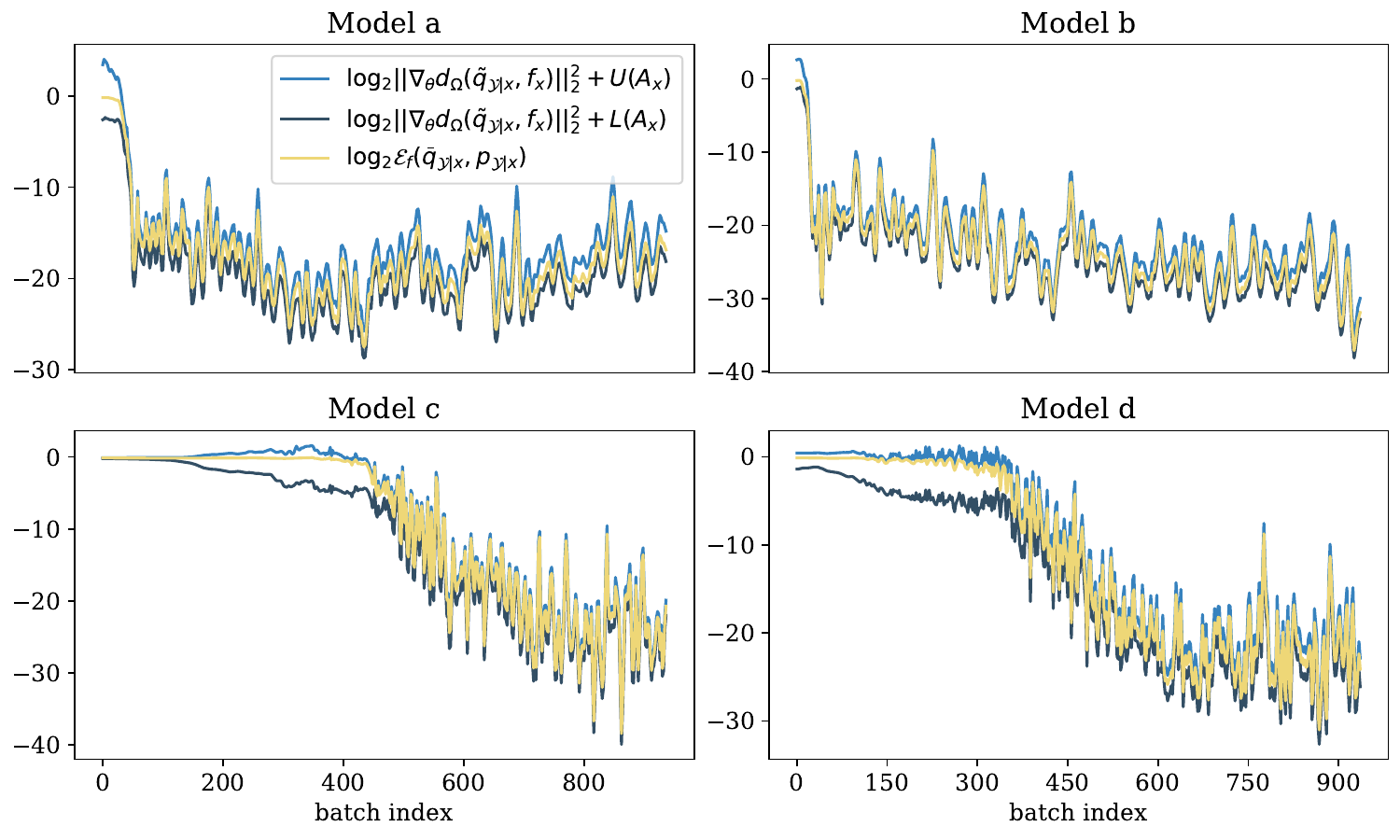}}
            \caption{Changes of bounds during the training process.}
            \label{fig:convergence_bound}
            \vspace{5mm} 
	\end{minipage}
	\begin{minipage}{1.0\linewidth}
		\centering
            \centerline{\includegraphics[width=1.0\columnwidth]{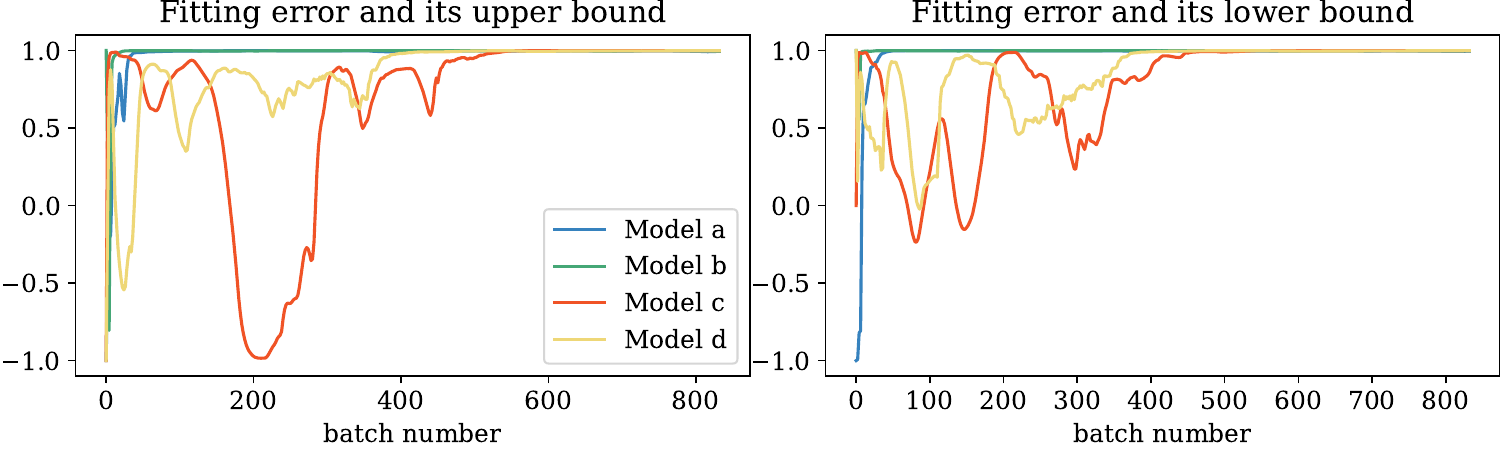}}
            \caption{Local Pearson correlation coefficient curves during the training process.}
            \label{fig:model_correlation}
	\end{minipage}
\vskip -0.1in
\end{figure}
As shown in Figure~\ref{fig:convergence_bound}, as training progresses, the variation in fitting error increasingly aligns with the trends of its upper and lower bounds. Concurrently, the gap between the bounds narrows. After the 500th epoch, the Pearson correlation coefficients between $\log \mathcal{E}_f(\tilde{q}_{\mathcal{Y}|x},p_{\mathcal{Y}|x})$ and its upper and lower bounds approach 1 across all models in Figure~\ref{fig:model_correlation}. These observations validate the Theorem~\ref{thm:station_error}, demonstrating that minimizing the fitting error is effectively achieved by controlling its upper and lower bounds in model training. 
\begin{figure}[ht]
\centering
    \centerline{\includegraphics[width=1.0\linewidth]{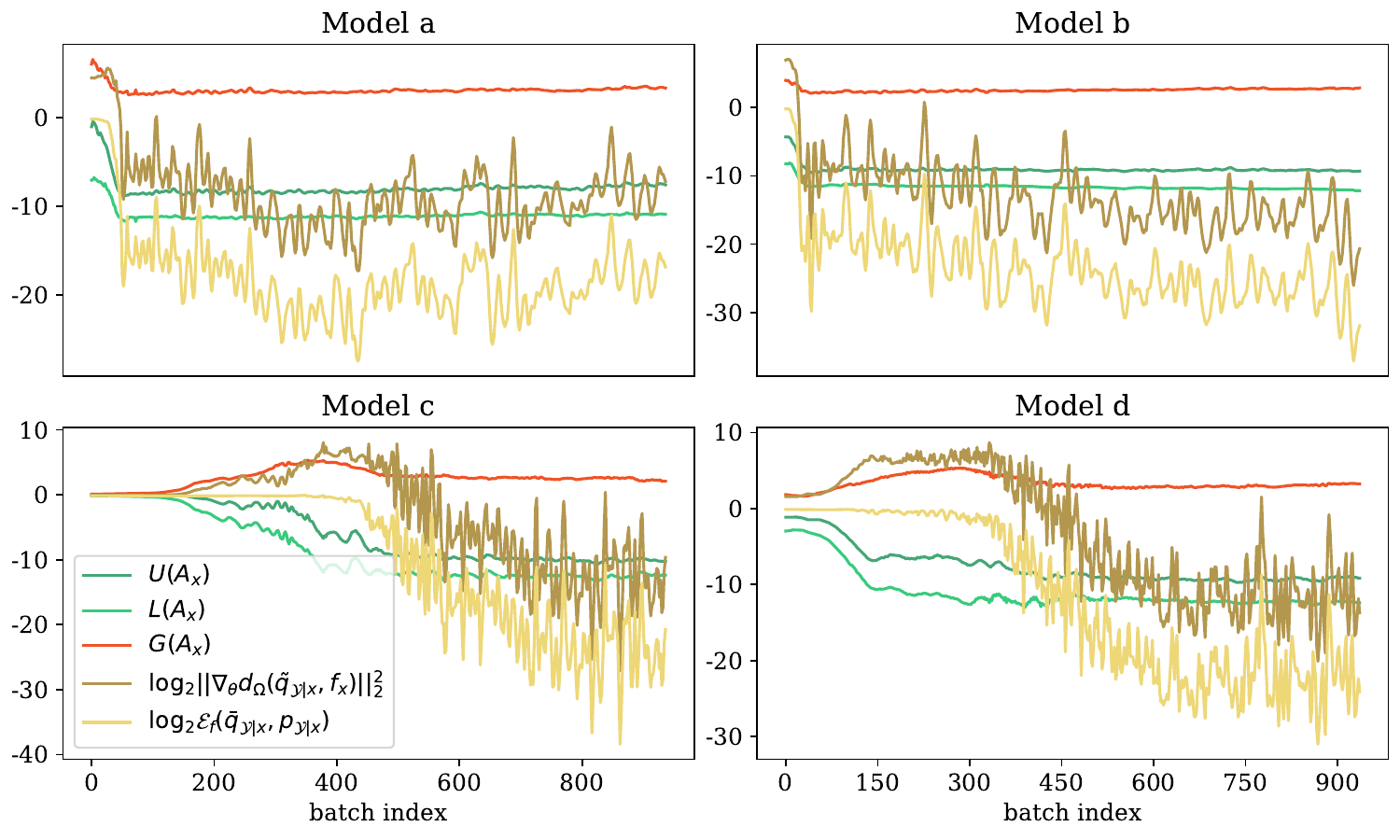}}
    \caption{Changes in indicators during the training process.}
    \label{fig:convergence_indicator}
\end{figure}
We further analyze the changes in structural error and gradient norm during the training process, as illustrated in Figure~\ref{fig:convergence_indicator}. 
Theorem~\ref{thm:station_error} implies that during model training, the gradient norm can control the fitting error only when the structural error remains constant. This conclusion is fully corroborated by the experimental results illustrated in Figure~\ref{fig:convergence_indicator}.

\subsection{Experimental Result on the Structural Error}  

\subsubsection{Initialization}
\begin{figure}[ht]  
\centering  
\centerline{\includegraphics[width=\columnwidth]{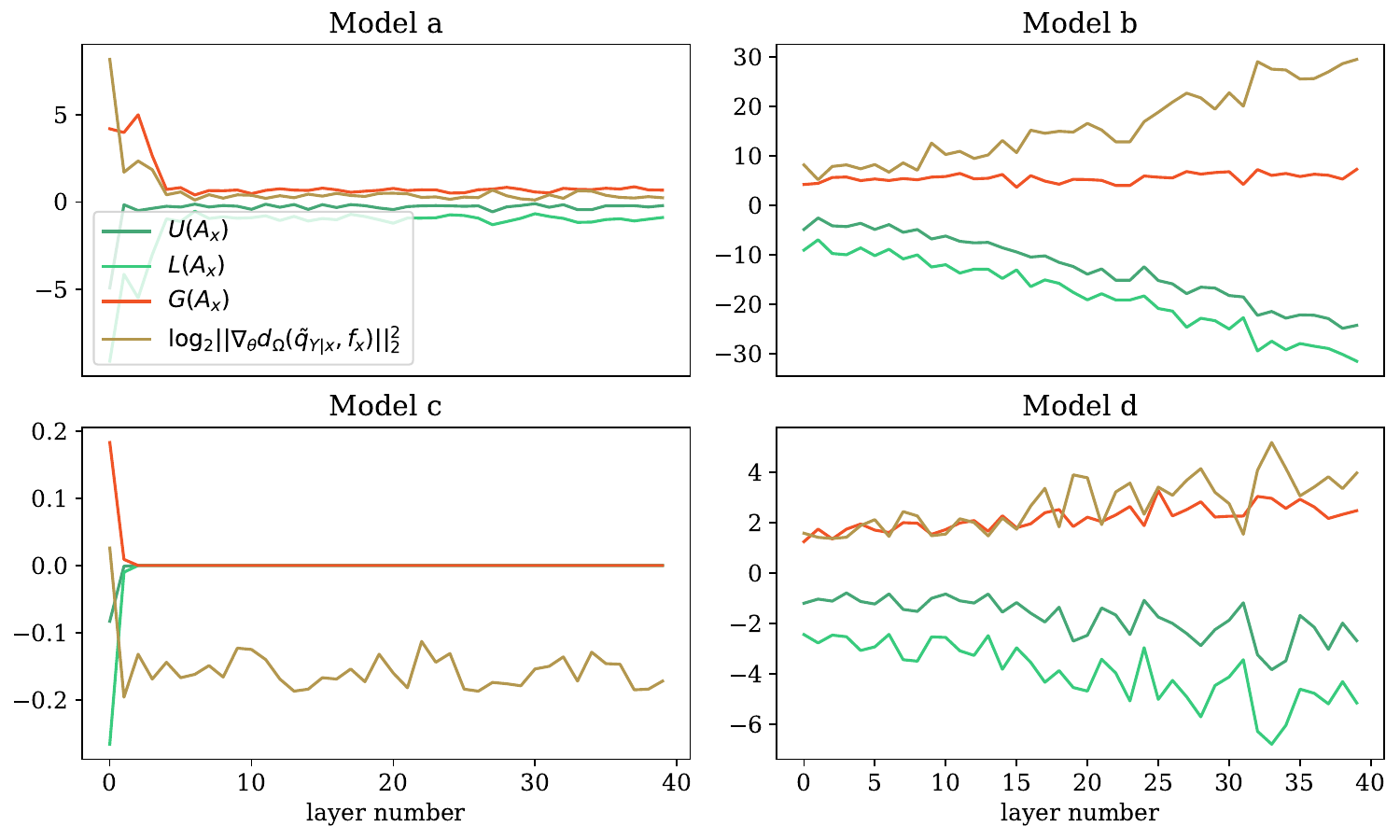}}  
\caption{Changes in indicators during the increase in model depth.}  
\label{fig:lay_init_indicator}  
\vskip -0.1in  
\end{figure}
We then examine the relationship between structural error, the number of parameters, and model structure (specifically, the use of skip connections) during the random initialization phase. 
The theoretical results in subsection~\ref{subsec:structural_error} suggest two key points:
\begin{enumerate}
\item Introducing skip connections increases the eigenvalues of the structural matrix, which in turn reduces the structural error.
\item During this phase, if skip connections are not used, the gradient independence condition holds true, implying that models with more parameters will have lower structural errors.
\end{enumerate}
We increase the depth of the models under the default initialization methods (He initialization) and observe changes in structural error. The results of this experiment are illustrated in Figure~\ref{fig:lay_init_indicator}.  

For Models a and c, which do not incorporate skip connections, as the number of layers increases, $U(A_x)$, $L(A_x)$ and $G(A_x)$ rapidly decrease, rapidly decrease, causing the structural error to decline and approach zero. This experimental result is fully consistent with the conclusions of~\ref{thm:number_para}. 
As shown in Figure~\ref{fig:lay_init_indicator}, skip connections in Model b and d have two key aspects: one is to amplify gradient values to prevent vanishing gradients, and the other is that the presence of skip connections can invalidate the Gradient Independence Condition. 
Therefore, both $U(A_x)$ and $L(A_x)$ decrease, while $G(A_x)$ increases. Given that $U(A_x) \gg L(A_x)$, the structural error is primarily determined by $U(A_x)$, leading to a reduction in structural error. 
According to ~\ref{thm:station_error}, a decrease in $U(A_x)$ facilitates the optimization of fitting error through the gradient norm. Consequently, the conclusions of this paper provide a new insight into understanding the role of skip connections in the optimization mechanisms of deep learning.

\subsubsection{Training Dynamics}
We increase the value of $ k $ from 0 to 5 in Models a and b and train these models to investigate the impact of parameter number on training dynamics. 
The experimental results are illustrated in Figure~\ref{fig:layer_nores_bound} and Figure~\ref{fig:layer_res_bound}. 
To examine the correlation between the fitting error and its bounds, we also computed the local Pearson correlation coefficients using a sliding window of length 50. The results are shown in Figure~\ref{fig:nores_layer_correlation} and Figure~\ref{fig:res_layer_correlation}.

\begin{figure}[htbp]
	\centering
	\begin{minipage}{1.0\linewidth}
		\centering
		\includegraphics[width=1.0\columnwidth]{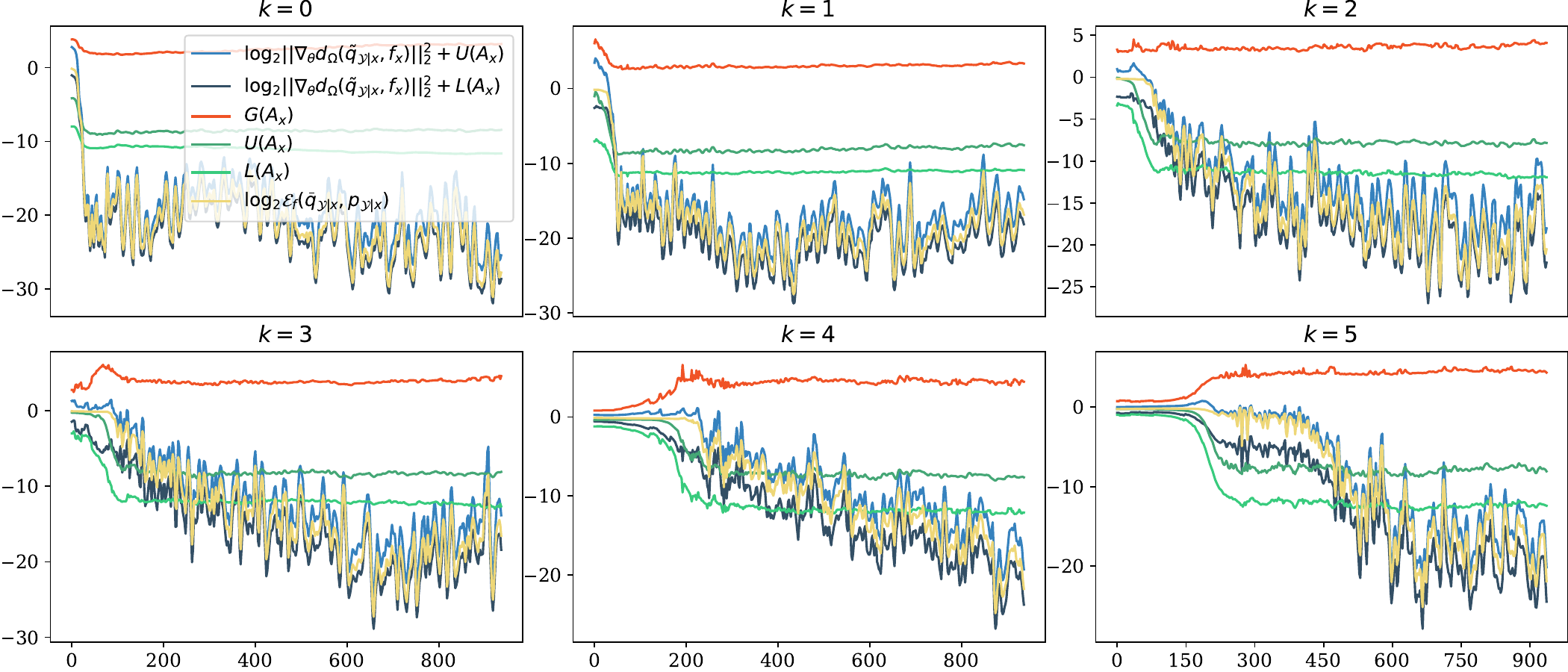}
		\caption{Bounds of model a with increasing blocks.}
		\label{fig:layer_nores_bound}
	\end{minipage}

	\begin{minipage}{1.0\linewidth}
		\centering
		\includegraphics[width=1.0\columnwidth]{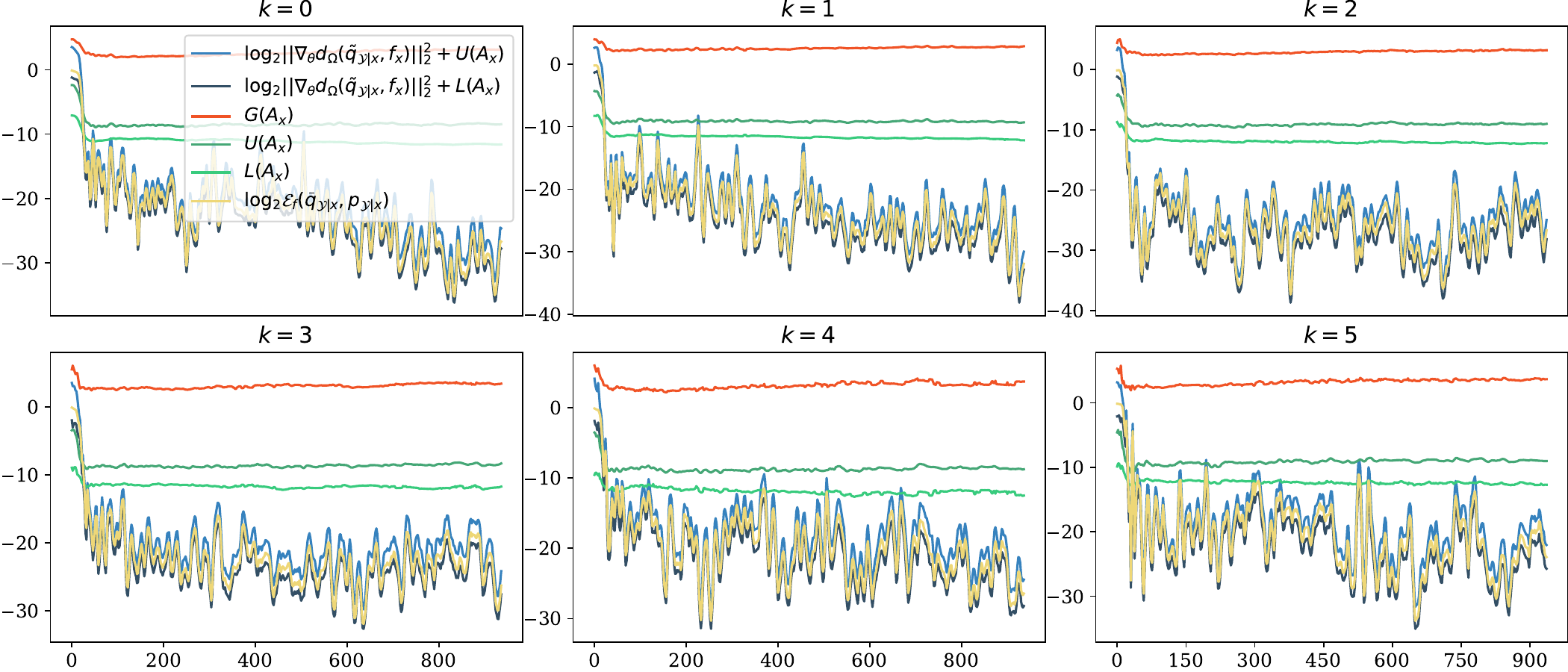}
		\caption{Bounds of model b with increasing blocks.}
		\label{fig:layer_res_bound}
	\end{minipage}
\end{figure}

\setlength{\parskip}{0.2cm plus4mm minus3mm}
\begin{figure}[htbp]
	\centering
	\begin{minipage}{0.9\linewidth}
		\centering
		\includegraphics[width=1.0\columnwidth]{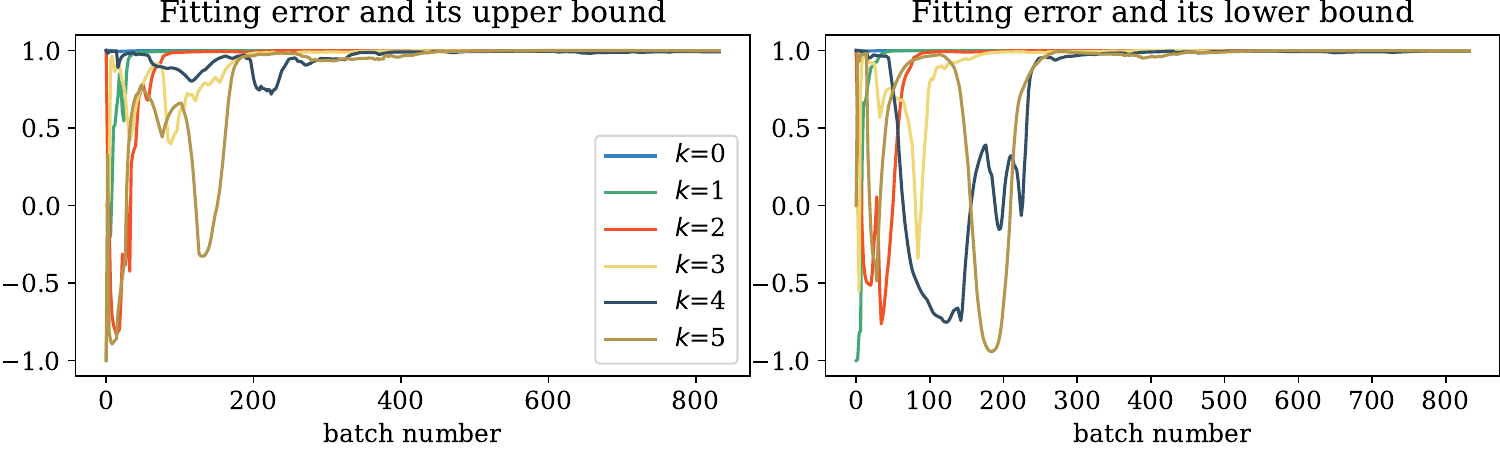}
		\caption{Local Pearson correlation coefficient curves of model a with increasing blocks.}
		\label{fig:nores_layer_correlation}
	\end{minipage}
	\begin{minipage}{0.9\linewidth}
		\centering
		\includegraphics[width=1.0\columnwidth]{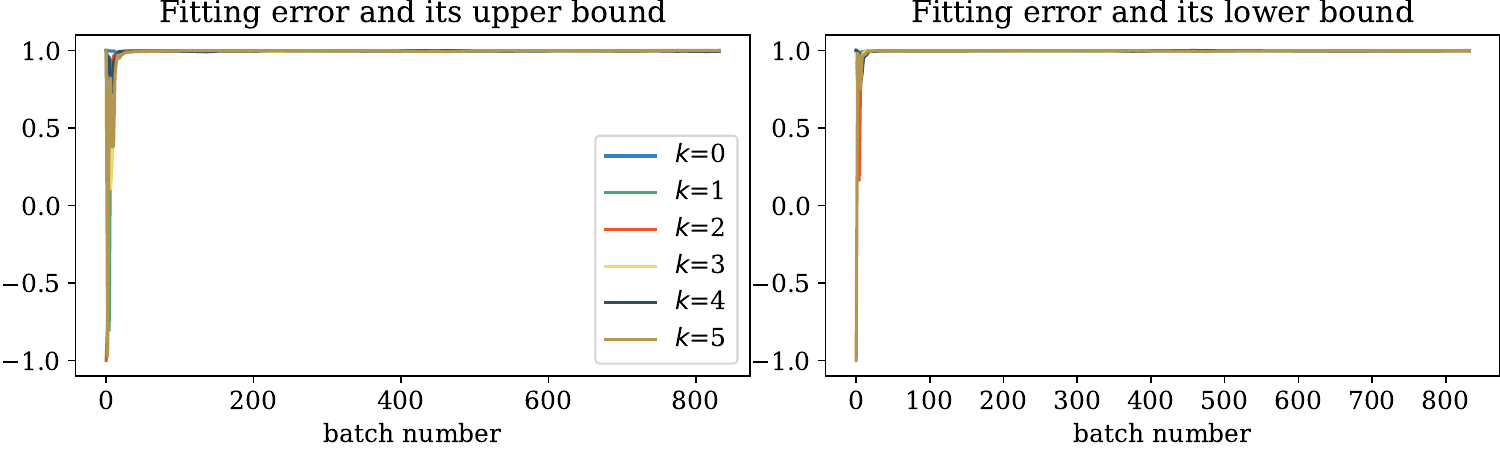}
		\caption{Local Pearson correlation coefficient curves of model b with increasing blocks.}
		\label{fig:res_layer_correlation}
	\end{minipage}
\end{figure}
Below, we present the analysis of the experimental results and the conclusions drawn from them.
\begin{itemize}
\item \textbf{In the absence of skip connections, increasing parameter number prolongs the time required for structural error to converge, thereby delaying the onset of fitting error reduction.} As illustrated in Figure~\ref{fig:layer_nores_bound}, as the number of layers $k$ increases, the timing at which the fitting error of Model a begins to decrease corresponds to the point when structural error converges. Increasing parameter number prolongs the time required for structural error to converge, thereby delaying the onset of fitting error reduction. This observation aligns with Theorem~\ref{thm:station_error}, which implies that the optimization of fitting error depends on controlling its upper and lower bounds. 
\item \textbf{In the absence of skip connections, increasing parameter number progressively fulfills the Gradient Independence Condition.} As illustrated in Figure~\ref{fig:layer_nores_bound}, as the number of layers $k$ increases, $U(A_x)$, $L(A_x)$, and $G(A_x)$ all approach zero, indicating that the structural error also tends toward zero. This phenomenon corroborates Theorem~\ref{thm:number_para}.
\item \textbf{Skip connections disrupt the initial Gradient Independence Condition but can accelerate the convergence of structural error.} At the onset of training, even with an increase in the number of layers, Model b does not exhibit a reduction in structural error. We attribute this primarily to the fact that skip connections disrupt the Gradient Independence Condition. Consequently, Theorem~\ref{thm:number_para} is no longer applicable. However, the presence of skip connections addresses the vanishing gradient problem and reduces the structural error, thereby accelerating the model's convergence as illustrated in Figure~\ref{fig:nores_layer_correlation} and Figure~\ref{fig:res_layer_correlation}. 
    
\end{itemize}

\section{Conclusion}
\label{sec:conclusion}
In conclusion, this study analyzes the optimization mechanisms of deep learning from the perspective of conditional distribution estimation and fitting. To ensure the generality of the results, the Fenchel-Young loss is utilized, which encompasses virtually all loss functions used in practical applications. This paper demonstrates that the non-convex optimization problems associated with training deep models can be addressed by minimizing the gradient norm and structural error. It shows that under typical configurations in deep learning, global optima can be approximated by their stationary points. Furthermore, the paper elucidates the influence mechanism of model parameter size on the optimization outcomes of deep learning under the assumption of gradient independence, providing theoretical insights into techniques such as over-parameterization and random parameter initialization.
In summary, this paper offers a novel theoretical perspective for understanding the optimization mechanisms in deep learning and validates its effectiveness through experiments. Additionally, while analyzing the impact of model size on the non-convex optimization capabilities, this study employs a relatively stringent assumption of gradient independence. Consequently, although the conclusions drawn based on this assumption provide valuable theoretical insights, their accuracy requires further examination. Exploring other more relaxed assumptions could be one direction for future work in this area.

\backmatter

\bmhead{Acknowledgements}

This work was supported in part by the National Key Research and Development Program of China under No. 2022YFA1004700, Shanghai Municipal Science and Technology, China Major Project under grant 2021SHZDZX0100, the National Natural Science Foundation of China under Grant Nos.~62403360, 72171172, 92367101, 62088101, iF open-funds from Xinghuo Eco and China Institute of Communications.

\bmhead{Supplementary information}




\section*{Declarations}

\subsection{Funding}
This work was supported in part by the National Key Research and Development Program of China under No. 2022YFA1004700, Shanghai Municipal Science and Technology, China Major Project under grant 2021SHZDZX0100, the National Natural Science Foundation of China under Grant Nos.~62403360, 72171172, 92367101, 62088101, iF open-funds from Xinghuo Eco and China Institute of Communications.

\subsection{Competing interests}
The authors declare that they have no known competing financial interests or personal relationships that could have appeared to influence the work reported in this paper.

\subsection{Ethics approval and consent to participate}
Not applicable.

\subsection{Data availability}
The datasets generated and/or analyzed in the present study are accessible via the GitHub repository at \href{https://yann.lecun.com/exdb/mnist/}{https://github.com/cvdfoundation/mnist} and the website at \href{https://yann.lecun.com/exdb/mnist/}{https://yann.lecun.com/exdb/mnist/}.

\subsection{Materials availability}
Not applicable.

\subsection{Code availability}
Code for preprocessing and analysis is provided if there is a clear need.

\subsection{Author contribution}
Binchuan Qi: Conceptualization, Methodology, Writing original draft. 
Wei Gong: Supervision, Review, Declarations. 
Li Li: Supervision, Review, Declarations.

\noindent

\begin{appendices}

\end{appendices}


\bibliography{sn-bibliography}

\end{document}